\documentclass[11pt]{article}
\usepackage{CJK,amsmath, amssymb, amsthm,color,mathrsfs}
\usepackage{setspace}
\usepackage{graphicx}
\usepackage{verbatim}
\usepackage{float}
\usepackage{color}

\usepackage{algorithm}
\usepackage{algpseudocode}

\newtheorem{theorem}{Theorem}[section]
\newtheorem{lemma}{Lemma}[section]
\newtheorem{definition}{Definition}[section]

\newtheorem{remark}{Remark}[section]

\begin{document}

\onehalfspacing

\title{A Robust Algorithm for Non-IID Machine Learning Problems with Convergence Analysis}

\author{Qing Xu\footnote{Huayuan Computing Technology(Shanghai) Co., Ltd.},~~ Xiaohua Xuan}

\date{}

\maketitle

\begin{abstract}
In this paper, we propose an improved numerical algorithm for solving minimax problems based on nonsmooth optimization, quadratic programming and iterative process. We also provide a rigorous proof of convergence for our algorithm under some mild assumptions, such as gradient continuity and boundedness. Such an algorithm can be widely applied in various fields such as robust optimization, imbalanced learning, etc.
\end{abstract}

\section{Introduction}

The classical machine learning framework relies on the assumption that the samples are independent and identically distributed (i.i.d.), which means that each sample has the same probability distribution as the others and all are mutually independent (see e.g. \cite{shen2021towards}). However, many real-world problems do not satisfy the i.i.d. assumption (e.g., when the data distribution changes over time or space, or when the samples are correlated with each other, etc., which may lead to biased or inconsistent estimators). Therefore, it is necessary to consider the problems that do not satisfy the i.i.d. assumption.

In~\cite{xu2019nonlinear}, We proposed a new framework for solving nonlinear regression problems without i.i.d. assumption. We formulated the nonlinear regression problem as a minimax problem with a max-mean loss function motivated by Peng~\cite{jin2016optimal}. 
$$\min_\theta \max_{1\leq j\leq N}\frac{1}{n_j}\sum_{l=1}^{n_j}(g^\theta(x_{jl}-y_{jl}))^2.$$

We then proposed a numerical algorithm to solve the above minimax problem. However, we did not give the convergence analysis of the algorithm.

In this paper, we propose a more efficient algorithm than the one in~\cite{xu2019nonlinear} and provide theoretical analysis on the convergence and the optimality conditions of the algorithm. Such an algorithm can be widely used in machine learning and deep learning problems.

\section{Preliminaries}

In this paper, we consider the following minimax problem
\begin{equation}\label{problem}
    \min_{x\in \mathbb{R}^n}\max_{1\leq j\leq N}f_j(x).
\end{equation}

In what follows, we always assume the following hypothesis.

\textbf{(H1)} There exists $M\in\mathbb{R}$ such that
$$f_j(x)\geq M,\quad \forall x\in \mathbb{R}^n,~1\leq j\leq N.$$

\textbf{(H2)} $f_j\in C^1(\mathbb{R}^n)$ and there exists a \textbf{modulus of continuity}\footnote{Recall that a modulus of continuity is an increasing function $w:[0,+\infty)\rightarrow [0,+\infty)$, vanishing at $0$ and continuous at $0$.} $w$ such that 
$$\|\nabla f_j(x)-\nabla f_j(y)\| \leq w( \|x-y\|),\quad \forall x,y\in \mathbb{R}^n .$$

Denote
$$ \Phi(x)=\max_{1\leq j\leq N}f_j(x)$$
and 
$$\Lambda(x)=\Big\{i|f_i(x) = \max_{1\leq j\leq N}f_j(x),i=1,2,\cdots,N \Big\}.$$

\bigskip

\begin{definition}
If for any direction $d\in \mathbb{R}^n$, the limit
$$\lim_{t\rightarrow 0+}\frac{g(x+td)-g(x)}{t}$$
exists, then we say $g$ is directional differentiable at $x$ and the directional derivative is denoted as
$$g'(x;d)=\lim_{t\rightarrow 0+}\frac{g(x+td)-g(x)}{t}.$$
\end{definition}

\bigskip

\begin{lemma}
For any direction $d\in \mathbb{R}^n$, the directional derivative of $\Phi$ exists and
$$\Phi'(x;d)=\max_{j\in\Lambda(x)}\langle \nabla f_j(x),d\rangle.$$
\end{lemma}

\begin{proof}
	For any $i\in \Lambda(x)$ and $j\notin \Lambda(x)$, 
	$$f_j(x)<f_i(x).$$
	Therefore, there exists $\delta >0$ such that for $\|y-x\|<\delta$,
	$$f_j(y)<f_i(y).$$
	Hence, for $\|y-x\|<\delta$,
	$$\Phi(y)=\max_{j\in \Lambda(x)}f_j(y).$$
	So for sufficiently small $t>0$, we have $x+td\in B(x;\delta)$ and
	\begin{align*}
	\Phi(x+td)-\Phi(x)&=\max_{j\in \Lambda(x)}f_j(x+td)-\max_{k\in \Lambda(x)}f_k(x)\\
		&=\max_{j\in \Lambda(x)}\left(f_j(x+td)-\max_{k\in \Lambda(x)}f_k(x)\right)\\
		&=\max_{j\in \Lambda(x)}(f_j(x+td)-f_j(x))\\
		&=\max_{j\in \Lambda(x)}\langle \nabla f_j(x+\theta_j td),td\rangle .
	\end{align*}
 
	Here, $\theta_j\in[0,1]$. Hence,
	\begin{align*}
		&\left|\frac{\Phi(x+td)-\Phi(x)}{t}-\max_{j\in \Lambda(x)}\langle \nabla f_j(x),d\rangle \right|\\
          =&\left|\max_{j\in \Lambda(x)}\langle \nabla f_j(x+\theta_j td),d\rangle  -\max_{j\in \Lambda(x)}\langle \nabla f_j(x),d\rangle \right| \\
          \leq &\left|\max_{j\in \Lambda(x)}\langle \nabla f_j(x+\theta_j td)-\nabla f_j(x),d\rangle   \right| \\
          \leq & w(\theta_j t\|d\|) \|d\| \\
          \leq & w( t\|d\|) \|d\|.
	\end{align*}
        
	Note that $\lim\limits_{t\rightarrow 0+} w( t\|d\|)=0$, we have that
	$$\Phi'(x;d)=\lim_{t\rightarrow 0+}\frac{\Phi(x+td)-\Phi(x)}{t}=\max_{j\in\Lambda(x)}\langle \nabla f_j(x),d\rangle.$$

\end{proof}
\begin{remark}
    In the above proof, we establish an inequality which is useful in the following sections.
    \begin{equation}\label{ineq}
    \left|\frac{\Phi(x+td)-\Phi(x)}{t}-\max_{j\in \Lambda(x)}\langle \nabla f_j(x),d\rangle \right|\leq  w( t\|d\|) \|d\|.
    \end{equation}

\end{remark}

\begin{lemma}
If $F$ is directional differentiable at $x$ for any direction $d$ and $F$ attain its minimum at $x$, then
\begin{equation}\label{e1}
F'(x;d)\geq 0,~\forall d\in\mathbb{R}^n.
\end{equation}
\end{lemma}
\begin{proof}
	For any $d$ and sufficiently small $t>0$, since $F$ attain its minimum at $x$, we have that
	$$F(x+td)-F(x)\geq 0\Rightarrow \frac{F(x+td)-F(x)}{t}\geq 0.$$
	Thus,
	$$F'(x;d)\geq 0.$$
\end{proof}
\begin{remark}
If \eqref{e1} holds, then $x$ is called a stationary point of $F$. If $F$ is convex and $x$ is a stationary point of $F$, then $F$ attain its minimum at $x$~(see e.g. \cite{dem1990introduction}).
\end{remark}

\section{Algorithm}

In this section, we formulate the main algorithm for solving problem \eqref{problem}.

\renewcommand{\thealgorithm}{X}
\begin{algorithm}[H]
    \caption{Main Algorithm}
    \begin{algorithmic}[1]
        \State \textbf{Initialization.} Set $k=0$, $x_0=0$, $\varepsilon = 10^{-8}$, $\delta=10^{-7}$, $c=0.5$, $\sigma=0.5$
        \While{true}
        \State$$\text{Set}\quad G=\nabla f(x_k)\in\mathbb{R}^{N\times n},~f=(f_1(x_k),\cdots,f_N(x_k))^T$$
        \State Suppose $\lambda$ is the solution of the following QP with gap tolerance $\delta$: \label{line1}
        $$\min_{\lambda}\left(\frac{1}{2}\lambda^TGG^T\lambda -f^T\lambda \right)$$
        $$\mathrm{s.t.}~\sum_{i=1}^{N}\lambda_i =1,\lambda_i\geq0 $$
        \State Set $p_k=-G^T\lambda$
        \If{$p_k =0$}
            \State Set $d_k=0$
       \Else
          \State Set $d_k=\frac{p_k}{\|p_k\|}$
       \EndIf

        \State Set $j=0$

        \While{true}\label{line2}
       
        \State Set $\alpha=\sigma^j$
        \If{$d_k=0$}
        \State \textbf{break}
        \EndIf
        \If{$\Phi(x_k+\alpha d_k)<\Phi(x_k)+c\alpha\Phi'(x_k;d_k)$}
            \State \textbf{break}
        \Else
        \State $j\leftarrow j+1$
        \EndIf
        \EndWhile\label{line3}
        \State Set $\alpha_k = \alpha$, $x_{k+1} = x_k  +\alpha_k d_k$
        \State $k\leftarrow k+1$
        
        \EndWhile
    \end{algorithmic}

\end{algorithm}

\begin{remark}
    The QP (quadratic programming) problem in line \ref{line1} can be solved by interior method, active set method, etc (see e.g. \cite{wright1999numerical}).
\end{remark}

\begin{remark}
We will show in the following sections that the \textbf{While} part (line \ref{line2} to line \ref{line3}) will terminate in finite steps. Thus, the above algorithm generate a finite sequence $\{x_k\}$.
\end{remark}

\begin{remark}
Algorithm X also works on minimax problems with other loss functions such as cross-entropy loss.
\end{remark}

\section{Convergence Analysis}

In this section, we provide the main convergence results of this paper. For the sake of simplicity, we will formulate the main results for a fixed $k\in\mathbb{N}$ and recall that

$$G=\nabla f(x_k)\in \mathbb{R}^{N\times n},f=(f_1(x_k),\cdots,f_N(x_k))^T.$$

\begin{theorem}\label{th1}
	If $\lambda$ is the solution of the following QP problem \eqref{dual1}$-$\eqref{dual2}:
	\begin{equation}\label{dual1}
	\min_\lambda \left(\frac{1}{2}\lambda^TGG^T\lambda -f^T\lambda \right)
	\end{equation}
	\begin{equation}\label{dual2}
	\mathrm{s.t.}~\sum_{i=1}^{N}\lambda_i =1,\lambda_i\geq0.
	\end{equation}
	Then $p=-G^T\lambda$ is the solution of problem \eqref{prim}$-$\eqref{rest}.
	\begin{equation}\label{prim}
		\min_{p,a}\quad \left(\frac{1}{2}\|p\|^2+a\right)
		\end{equation}
		\begin{equation}\label{rest}
		\text{s.t.}~f_j(x_k)+\langle \nabla f_j(x_k),p\rangle\leq a,~\forall~1\leq j\leq N.
		\end{equation}
\end{theorem}
	
\begin{proof}
	Consider the Lagrangian
	$$L(p,a;\lambda)=\frac{1}{2}\|p\|^2+a+\sum_{j=1}^{N}\lambda_j(f_j(x_k)+\langle \nabla f_j(x_k),p\rangle - a).$$	
	It is easy to verify that problem (\ref{prim})$-$(\ref{rest}) is equivalent to the following minimax problem.
	$$\min_{p,a}\max_{\lambda \geq 0}L(p,a;\lambda).$$	
	Since $L(\cdot,\cdot;\lambda)$ is convex and $L(p,a;\cdot)$ is linear, by Sion's minimax theorem~\cite{sion1958general}, we have that
	$$\min_{p,a}\max_{\lambda \geq 0}L(p,a;\lambda)=\max_{\lambda \geq 0}\min_{p,a}L(p,a;\lambda).$$	
	Set $e=(1,1,\cdots,1)^T$, the above problem is equivalent to
	\begin{equation}\label{maxmin}
	    \max_{\lambda \geq 0}\min_{p,a}\left(\frac{1}{2}\|p\|^2+a+\lambda^T(f+Gp-ae)\right).
	\end{equation}
 
	Note that
	\begin{equation*}
		\frac{1}{2}\|p\|^2+a+\lambda^T(f+Gp-ae)=\frac{1}{2}\|p\|^2+\lambda^T(f+Gp)+a(1-\lambda^Te).
	\end{equation*}	
	If $1-\lambda^Te\neq 0$, then the inner minimum of \eqref{maxmin} is $-\infty$. Thus, we must have $1-\lambda^Te = 0$ when the outer maximum is attained. The problem is converted to
            \begin{equation}\label{simpleform}
	\max_{\lambda_i \geq 0,\sum\limits_{i=1}^N\lambda_i=1}\min_{p}\left(\frac{1}{2}\|p\|^2+\lambda^TGp+\lambda^Tf\right).
            \end{equation}
	The inner minimum of \eqref{simpleform} is achieved when $p=-G^T\lambda$ and the above problem is reduced to
	$$\min_\lambda \left(\frac{1}{2}\lambda^TGG^T\lambda -f^T\lambda \right)$$
	$$\mathrm{s.t.}~\sum_{i=1}^{N}\lambda_i =1,\lambda_i\geq0.$$
        Thus, we finish the proof.
\end{proof}

Consider the following optimization problem

\begin{equation}\label{kk}
	\min_{p\in\mathbb{R}^n}\left\{\max_{1\leq j\leq N}\{f_j(x_k)+\langle \nabla f_j(x_k),p\rangle\}+\frac{1}{2}\|p\|^2\right\},
\end{equation}
It is obvious that problem \eqref{kk} is equivalent to problem (\ref{prim})$-$(\ref{rest}).

\begin{theorem}\label{th2}
If $\lambda$ is the solution of problem \eqref{dual1}$-$\eqref{dual2}, and $p=-G^T\lambda$. Then
	$$\Phi'(x_k;p)\leq -\frac{1}{2}\|p\|^2.$$ 
\end{theorem}

\begin{proof}
	For $0<t<1$,
	\begin{align*}
	&\Phi(x_k+tp)-\Phi(x_k)\\
	=&\max_{1\leq j\leq N}\{f_j(x_k+tp)-\Phi(x_k)\}\\
	=&\max_{1\leq j\leq N}\{f_j(x_k)+\langle \nabla f_j(x_k+\theta_jtp),tp\rangle -\Phi(x_k)\}\\
	=&\max_{1\leq j\leq N}\{f_j(x_k)+\langle \nabla f_j(x_k),tp\rangle -\Phi(x_k) + \langle \nabla f_j(x_k+\theta_jtp)-\nabla f_j(x_k),tp\rangle\}\\
        \leq& \max_{1\leq j\leq N}\{f_j(x_k)+\langle \nabla f_j(x_k),tp\rangle -\Phi(x_k)\} + \max_{1\leq j\leq N}\{\langle \nabla f_j(x_k+\theta_jtp)-\nabla f_j(x_k),tp\rangle\}\\
        \leq& \max_{1\leq j\leq N}\{f_j(x_k)+\langle \nabla f_j(x_k),tp\rangle -\Phi(x_k)\} + w(t\|p\|)t\|p\|\\
	=&\max_{1\leq j\leq N}\{t(f_j(x_k)+\langle \nabla f_j(x_k),p\rangle -\Phi(x_k))+(1-t)(f_j(x_k)-\Phi(x_k))\}+w(t\|p\|)t\|p\|\\
	&\quad\quad\quad \left(\text{Note that } f_j(x_k)\leq \Phi(x_k)=\max_{1\leq j\leq N}f_j(x_k)\right)\\
	\leq &t\max_{1\leq j\leq N}\{f_j(x_k)+\langle \nabla f_j(x_k),p\rangle -\Phi(x_k)\}+w(t\|p\|)t\|p\|.
	\end{align*}	
	Since $\lambda$ is the solution of problem \eqref{dual1}$-$\eqref{dual2}, $p$ is the solution of problem (\ref{prim})$-$(\ref{rest}), and therefore is also the solution of problem \eqref{kk}. We have that
	\begin{align*}
	&\max_{1\leq j\leq N}\left\{f_j(x_k)+\langle \nabla f_j(x_k),p\rangle+\frac{1}{2}\|p\|^2 \right\}\\
	\leq &\max_{1\leq j\leq N}\left\{f_j(x_k)+\langle \nabla f_j(x_k),0\rangle+\frac{1}{2}\|0\|^2 \right\}\\
	=& \max_{1\leq j\leq N}\{f_j(x_k)\}\\
	=&\Phi(x_k).
	\end{align*}	
	Therefore,
	\begin{align*}
	&\max_{1\leq j\leq N}\{f_j(x_k)+\langle \nabla f_j(x_k),p\rangle -\Phi(x_k)\}\leq -\frac{1}{2}\|p\|^2.\\
	\Rightarrow~ &\Phi(x_k+tp)-\Phi(x_k)\leq -\frac{1}{2}t\|p\|^2+w(t\|p\|)t\|p\|.\\
	\Rightarrow~&\frac{\Phi(x_k+tp)-\Phi(x_k)}{t}\leq -\frac{1}{2}\|p\|^2+w(t\|p\|)\|p\|.\\
	\Rightarrow~&\lim_{t\rightarrow0+}\frac{\Phi(x_k+tp)-\Phi(x_k)}{t}\leq -\frac{1}{2}\|p\|^2.
	\end{align*}
	Hence,
	$$\Phi'(x_k;p)\leq -\frac{1}{2}\|p\|^2.$$ 
\end{proof}

Next theorem states that if $d_k\neq 0$, then it is a descent direction for $\Phi$.
\begin{theorem}\label{th3}
	If $d_k\neq 0$, then
	$$\Phi'(x_k;d_k)<0.$$
\end{theorem}
\begin{proof}
	Since $d_k=\beta p$ with $\beta>0$. Hence,
	$$\Phi'(x_k;d_k)=\max_{j\in\Lambda(x_k)}\langle \nabla f_j(x),d_k\rangle =\beta \max_{j\in\Lambda(x_k)}\langle \nabla f_j(x),p\rangle=\beta \Phi'(x_k;p)<0.$$

\end{proof}

\begin{theorem}\label{th4}
	The \textbf{While} part (line \ref{line2} to line \ref{line3}) in the algorithm will terminate in finite steps.
\end{theorem}
\begin{proof}
        If $d_k=0$, then it terminates for one step. If $d_k\neq 0$, it suffices to show that for sufficiently small $t>0$, we have that
	$$\Phi(x_k+td_k)<\Phi(x_k)+ct\Phi'(x_k;d_k).$$
	In fact, according to \eqref{ineq}, 
	$$|\Phi(x_k+td_k)-\Phi(x_k)-t\Phi'(x_k;d_k) |\leq w( t\|d_k\|) t\|d_k\|.$$
	Hence,
	\begin{align*}
		\Phi(x_k+td_k)-\Phi(x_k) &\leq w( t\|d_k\|) t\|d_k\|+t\Phi'(x_k;d_k)\\
		&=ct\Phi'(x_k;d_k)+w( t\|d_k\|) t\|d_k\|+(1-c)t\Phi'(x_k;d_k)\\
		&=ct\Phi'(x_k;d_k)+t\left(w( t\|d_k\|) \|d_k\|+(1-c)\Phi'(x_k;d_k)\right).
	\end{align*}
	By Theorem \ref{th3}, $\Phi'(x_k;d_k)<0$. On the other hand, $\lim\limits_{t\rightarrow0}w( t\|d_k\|) \|d_k\|=0$. Thus for sufficiently small $t>0$, we have that
	$$w( t\|d_k\|) \|d_k\|+(1-c)\Phi'(x_k;d_k)<0.$$
	Therefore,
	$$\Phi(x_k+td_k)<\Phi(x_k)+ct\Phi'(x_k;d_k).$$
\end{proof}

\begin{theorem}\label{th5}
	Under \textbf{(H1)} and \textbf{(H2)}, we have that 
$$\lim_{k\rightarrow\infty}\Phi'(x_k;d_k)=0.$$

\end{theorem}

\begin{proof}

       If there exists $m$ such that $d_m=0$, then for $k\geq m$, $d_k=0$, and hence $\Phi'(x_k;d_k)=0$.
        Next, We assume that for any $k$, $d_k\neq 0$. According to Theorem \ref{th4}, it is easy to verify that
	$$M\leq \Phi(x_{k+1})\leq \Phi(x_k).$$
	We have that
	$$\lim_{k\rightarrow\infty} (\Phi(x_{k+1})-\Phi(x_k))=0.$$
	Note that
	$$\Phi(x_k+\alpha_kd_k)-\Phi(x_k)\leq c\alpha_k\Phi'(x_k;d_k)<0.$$
	Hence,
	$$\lim_{k\rightarrow\infty}\alpha_k\Phi'(x_k;d_k)=0.$$

	If $\Phi'(x_k;d_k)$ not tends to $0$, then there exists an infinite subset $\Gamma\subset\mathbb{N}$ and $\beta<0$ such that 
	$$\sup_{k\in \Gamma}\Phi'(x_k;d_k)<\beta.$$
Without loss of generality, we suppose $\Gamma=\mathbb{N}$. Then we must have that $\alpha_k\rightarrow 0$.  Again, without loss of generality, we assume that $\alpha_k<1$.
	Therefore, 
	$$\Phi(x_k+\sigma^{-1}\alpha_kd_k)-\Phi(x_k)> c\sigma^{-1}\alpha_k\Phi'(x_k;d_k).$$

	Set $$\Lambda_k = \Lambda(x_k)=\left\{i|f_i(x_k)=\max_{1\leq j\leq N}f_j(x_k)\right\}.$$
	Then
	\begin{align*}
		&\quad \Phi(x_k+\sigma^{-1}\alpha_kd_k)-\Phi(x_k) \\
		&=\max_{j\in\Lambda_k}(f_j(x_k+\sigma^{-1}\alpha_kd_k)-f_j(x_k))\\
		&=\max_{j\in\Lambda_k}(\langle \nabla f_j(x_k+\theta_k\sigma^{-1}\alpha_kd_k), \sigma^{-1}\alpha_kd_k\rangle\\
		&=\max_{j\in\Lambda_k}(\langle \nabla f_j(x_k+\theta_k\sigma^{-1}\alpha_kd_k)-\nabla f_j(x_k), \sigma^{-1}\alpha_kd_k\rangle + \langle \nabla f_j(x_k), \sigma^{-1}\alpha_kd_k\rangle)\\
		&\leq w(\|\theta_k\sigma^{-1}\alpha_kd_k\|)\|\sigma^{-1}\alpha_kd_k\|+  \max_{j\in\Lambda_k}\langle \nabla f_j(x_k), \sigma^{-1}\alpha_kd_k\rangle\\
		&\leq w(\sigma^{-1}\alpha_k)\sigma^{-1}\alpha_k+  \max_{j\in\Lambda_k}\langle \nabla f_j(x_k), \sigma^{-1}\alpha_kd_k\rangle.
	\end{align*}

	Therefore,
	$$w(\sigma^{-1}\alpha_k)\sigma^{-1}\alpha_k+  \max_{j\in\Lambda_k}\langle \nabla f_j(x_k), \sigma^{-1}\alpha_kd_k\rangle> c\sigma^{-1}\alpha_k\Phi'(x_k;d_k).$$
	
	$$w(\sigma^{-1}\alpha_k)+  (1-c)\beta > 0.$$
	Note that Let $\alpha_k\rightarrow 0$, we have that
	$$(1-c)\beta \geq 0,$$
	which is a contradiction.

\end{proof}

\begin{theorem}\label{th6}
	Under $\textrm{(H1)}$ and $\textrm{(H2)}$, Suppose $\bar{x}$ is an accumulation point of $\{x_k\}$, then $\bar{x}$ is stationary.
\end{theorem}

\begin{proof}
     Without loss of generality, we assume that 
    $$\lim_{k\rightarrow\infty }x_k=\bar{x}.$$
    Denote by 
    $$\Lambda = \Lambda(\bar{x}).$$
    Suppose there exists $A,B>0$ such that
    $$\max_{i\notin \Lambda} f_i(\bar{x}) + A \leq \max_{j\in \Lambda} f_j(\bar{x}),~\max_{1\leq j\leq N}\|\nabla f_j(\bar{x})\|\leq B.$$
    Then for any $0<\varepsilon \leq  \frac{A}{4B}$, when $\|q\| = \varepsilon$, we have that for $i\notin \Lambda$ and $j\in \Lambda$,
    $$f_i(\bar{x})+\langle \nabla f_i(\bar{x}),q\rangle +\frac{A}{2}\leq f_j(\bar{x})+\langle \nabla f_j(\bar{x}),q\rangle.$$
    Since for each $i=1,2,\cdots,N$,
    $$\lim_{k\rightarrow\infty }f_i(x_k) = f_i(\bar{x}),~\lim_{k\rightarrow\infty }\nabla f_i(x_k) = \nabla f_i(\bar{x}).$$
    Thus, there exists $m>0$ such that when $k\geq m$,
     $$f_i(x_k)+\langle \nabla f_i(x_k),q\rangle \leq f_j(x_k)+\langle \nabla f_j(x_k),q\rangle.$$
    
    On the other hand, since $p$ is the solution of problem \eqref{kk},
    $$\max_{1\leq j\leq N}\{f_j(x_k)+\langle \nabla f_j(x_k),p\rangle\}+\frac{1}{2}\|p\|^2\leq \max_{1\leq j\leq N}\{f_j(x_k)+\langle \nabla f_j(x_k),q\rangle\}+\frac{1}{2}\|q\|^2$$
    For each $j\in\Lambda(x_k)$, 
        $$\Phi(x_k)+\langle \nabla f_j(x_k),p\rangle+\frac{1}{2}\|p\|^2\leq \max_{1\leq j\leq N}\{f_j(x_k)+\langle \nabla f_j(x_k),q\rangle\}+\frac{1}{2}\|q\|^2.$$
        $$\Phi(x_k)+\langle \nabla f_j(x_k),\frac{p}{\|p\|}\rangle\|p\|+\frac{1}{2}\|p\|^2\leq \max_{1\leq j\leq N}\{f_j(x_k)+\langle \nabla f_j(x_k),q\rangle\}+\frac{1}{2}\|q\|^2$$

        $$\Phi(x_k)-\varepsilon \|p\|+\frac{1}{2}\|p\|^2\leq \max_{1\leq j\leq N}\{f_j(x_k)+\langle \nabla f_j(x_k),q\rangle\}+\frac{1}{2}\|q\|^2.$$

         $$\Phi(x_k)-\frac{1}{2}\|\varepsilon\|^2\leq \max_{1\leq j\leq N}\{f_j(x_k)+\langle \nabla f_j(x_k),q\rangle\} +\frac{1}{2}\|q\|^2.$$
         Let $k\rightarrow\infty$, we have that
         $$\Phi(\bar{x})-\frac{1}{2}\|\varepsilon\|^2\leq \max_{1\leq j\leq N}\{f_j(\bar{x})+\langle \nabla f_j(\bar{x}),q\rangle\} +\frac{1}{2}\|q\|^2.$$
           $$\Phi(\bar{x})-\frac{1}{2}\|\varepsilon\|^2\leq \max_{ j\in\Lambda }\{f_j(\bar{x})+\langle \nabla f_j(\bar{x}),q\rangle\} +\frac{1}{2}\|q\|^2.$$
        $$\Phi(\bar{x})-\frac{1}{2}\|\varepsilon\|^2\leq \max_{ j\in\Lambda }f_j(\bar{x})+\max_{ j\in\Lambda }\langle \nabla f_j(\bar{x}),q\rangle +\frac{1}{2}\|q\|^2.$$
         $$-\varepsilon\leq \max_{ j\in\Lambda }\langle \nabla f_j(\bar{x}),\frac{q}{\|q\|}\rangle .$$
         So for each $c$ with $\|c\|=1$, 
         $$-\varepsilon\leq \max_{ j\in\Lambda }\langle \nabla f_j(\bar{x}),c\rangle .$$
         Let $\varepsilon \rightarrow 0$, we have that
         $$ \max_{ j\in\Lambda }\langle \nabla f_j(\bar{x}),c\rangle \geq 0.$$
         Therefore, $\bar{x}$ is a stationary point of $\Phi$.
\end{proof}

\begin{theorem}
    Under $\textrm{(H1)}$ and $\textrm{(H2)}$, suppose each $f_j$ is strictly convex, then the $\{x_k\}$ generated in the algorithm converges to the global minimum of $\Phi$.
\end{theorem}

\begin{proof}
    Since $\Phi$ is strictly convex, and $\{\Phi(x_k)\}$ is decreasing, $\{x_k\}$ must be bounded. Since any accumulation $\bar{x}$ is a stationary point of $\Phi$ and the stationary point of $\Phi$ is unique by the convexity. Thus, 
    $$x_k\rightarrow \bar{x}.$$
    And the stationary point of $\Phi$ is unique.
\end{proof}

\section{Conclusions}
In this paper, we propose an improved numerical algorithm for solving minimax problems based on nonsmooth optimization, quadratic programming and iterative process. We also provide theoretical analysis on the convergence and the optimality conditions of the algorithm. Such an algorithm can be widely applied in various fields such as robust optimization, imbalanced learning, etc.

\bibliography{eig}{}
\bibliographystyle{plain}

\end{document}